\documentclass[a4paper,fleqn]{cas-sc}

\usepackage[authoryear,longnamesfirst]{natbib}
\usepackage{setspace}

\def\tsc#1{\csdef{#1}{\textsc{\lowercase{#1}}\xspace}}
\tsc{WGM}
\tsc{QE}
\tsc{EP}
\tsc{PMS}
\tsc{BEC}
\tsc{DE}
\usepackage{amsmath,amssymb,amsfonts,amsthm,bm,mathtools, subcaption}
\usepackage{setspace}
\usepackage{booktabs}
\usepackage{enumitem}
\usepackage{hyperref}
\usepackage{microtype}
\usepackage{xcolor, cancel}
\usepackage[normalem]{ulem}
\newcommand{\R}{\mathbb{R}}

 \newtheorem{theorem}{Theorem}
 
 \newdefinition{rmk}{Remark}
 \newproof{pf}{Proof}
 \newproof{pot}{Proof of Theorem \ref{thm2}}
\newtheorem{ass}[subsection]{A-}

\begin{document}
\let\WriteBookmarks\relax
\def\floatpagepagefraction{1}
\def\textpagefraction{.001}
\shorttitle{Variational REML}
\shortauthors{D.Thakur}
%\begin{frontmatter}

\title [mode = title]{Variational Approximated Restricted Maximum Likelihood Estimation for Spatial Data}                     
\tnotemark[1]

% \tnotetext[1]{This document is the results of the research
%    project funded by the National Science Foundation.}

% \tnotetext[2]{The second title footnote which is a longer text matter
%    to fill through the whole text width and overflow into
%    another line in the footnotes area of the first page.}

\author[1,2]{Debjoy Thakur}[type=editor,
                        orcid=0000-0001-5514-5707]
\cormark[1]
\ead{debjoythakur@outlook.com, debjoy.thakur@ahduni.edu.in}
% \ead[url]{www.jkkrishnan.in}

\credit{Conceptualization of this study, Methodology, Software}

\affiliation[1]{organization={Washington University in St. Louis}, country={USA}}
\affiliation[2]{organization={Ahmedabad University}, country={India}}

% \author[2,4]{Han Thane}[style=chinese]

% \author[2,3]{William {J. Hansen}}[%
%    role=Co-ordinator,
%    suffix=Jr,
% \ead{wjh@example.org}
% \ead[URL]{https://www.university.org}

% \credit{Data curation, Writing - Original draft preparation}

% \affiliation[2]{organization={World Scientific University},
%                 addressline={Street 29}, 
%                 postcode={1011 NX}, 
%                 postcodesep={}, 
%                 city={Amsterdam},
%                 country={The Netherlands}}

% \author[1,3]{T. Rafeeq}
% \cormark[2]
% \fnmark[1,3]
% \ead{t.rafeeq@example.in}
% \ead[URL]{www.campus.in}

% \affiliation[3]{organization={University of Intelligent Studies},
%                 addressline={Street 15}, 
%                 city={Jabaldesh},
%                 postcode={825001}, 
%                 state={Orissa}, 
%                 country={India}}

\cortext[cor1]{Corresponding author}
% \cortext[cor2]{Principal corresponding author}
% \fntext[fn1]{This is the first author footnote, but is common to third
% \fntext[fn2]{Another author footnote, this is a very long footnote and
%   it should be a really long footnote. But this footnote is not yet
%   sufficiently long enough to make two lines of footnote text.}

% \nonumnote{This note has no numbers. In this work we demonstrate $a_b$
%   the formation Y\_1 of a new type of polariton on the interface
%   between a cuprous oxide slab and a polystyrene micro-sphere placed

\begin{abstract}
This research considers a scalable inference for spatial data modeled through Gaussian intrinsic conditional autoregressive (ICAR) structures. The classical estimation method, restricted maximum likelihood (REML), requires repeated inversion and factorization of large, sparse precision matrices, which makes this computation costly. To sort this problem out, we propose a variational restricted maximum likelihood (VRMLE) framework that approximates the intractable marginal likelihood using a Gaussian variational distribution. By constructing an evidence lower bound (ELBO) on the restricted likelihood, we derive a computationally efficient coordinate-ascent algorithm for jointly estimating the spatial random effects and variance components. In this article, we theoretically establish the monotone convergence of ELBO and mathematically exhibit that the variational family is exact under Gaussian ICAR settings, which is an indication of nullifying approximation error at the posterior level. We empirically establish the supremacy of our VRMLE over MLE and INLA.
\end{abstract}

% \includegraphics{figs/cas-grabs.pdf}

% \begin{highlights}
% \item Research highlights item 1
% \item Research highlights item 2
% \item Research highlights item 3
% \end{highlights}


\maketitle
\section{Introduction}
In recent years, variational approximation has become a well-known method for approximating likelihoods in the statistical community and among computer scientists. For an overview of variational approximation, readers are encouraged to consult \cite{bishop2006pattern}, \cite{titterington2004bayesian}, and \cite{jordan2004graphical}; for a scalable REML model in the presence of crossed random effects, please see \cite{ghosh2022scalable}. \cite{hall2011theory} has established the theoretical justification behind likelihood-based inference using Gaussian variational approximation for the Poisson Mixed model. \cite{ormerod2012gaussian} first introduced the variational approximation for generalized linear mixed model (GLMM) involving a Gaussian approximation to the conditional distribution of random effect conditioning on the response, and they have established Gaussian variational approximation as a potential alternative to the Laplace approximation for fast non-Monte Carlo analysis of GLMM. \cite{challis2013gaussian} has proposed a novel Gaussian Kullback-Leibler (G-KL) variational approximation for Bayesian GLMM. \cite{tran2017variational} has extended the variational Bayes approximation for intractable likelihood. \cite{tan2018gaussian} has solved the problem of learning Gaussian variational approximation by incorporating sparsity in the precision matrix to approximate the posterior distribution for a high-dimensional parameter space. \cite{alquier2020concentration} has established an oracle inequality related to the quality of variational Bayesian (VB) approximation to the prior and the structure of approximation of the posterior in a tractable family. \cite{bhattacharya2025convergence} has established the convergence of the coordinate ascent variational inference (CAVI) algorithm, focusing on the two-block case by applying mean-field variational inference towards optimizing the KL divergence, and they have also established global or local exponential convergence for CAVI. \cite{goplerud2025partially} has theoretically justified that naive usage of variational inference for the case of mixed models, the traditional mean-field variational inference underestimates the posterior uncertainty in high-dimensional scenarios, and they have theoretically established that a partially factorized family maintains a trade-off between the computational complexity and approximation quality.

For the spatial framework, there are very few research articles who are dealing with variational inference. Among them, the first \cite{ren2011variational} has established that VB is a potential alternative to MCMC for approximating the posterior distribution for spatial models. \cite{wu2018fast} has proposed a hybrid mean-field variational Bayes (MFVB), which provides accurate results assuming posterior independence, although it underestimates the posterior variances. Under unobserved spatial dependence \cite{bansal2021fast} proposed a VB method for posterior inference for overdispersed spatial data, and in this scenario, Po\'lya-Gamma augmentation has a significant role in capturing posterior dependency. Similarly, \cite{parker2022computationally} uses Po\'lya-Gamma mixtures to bypass computationally expensive expectations for modeling spatial binary and count data. Recently \cite{lee2025scalable} has proposed a computationally scalable VB approach for a spatial generalized linear model, assuming the random effect as Gaussian under the assumption of a parametric covariance function and low-rank approximations.

To our knowledge, we are the first to work on a variational analogue of the REML method for the Spatial ICAR model. This approach is motivated by \cite{hall2011theory} and \cite{ormerod2012gaussian}'s Gaussian variational approximation. This work provides some novel insights into the modeling of spatial data as follows:
\begin{enumerate}
    \item We introduce a novel variational analog of REML by bridging between frequentist estimation and modern variational inference. In this approach, we optimize the tractable ELBO instead of computationally costly matrix factorization like INLA or a full likelihood-based approach.
    \item We derive closed-form updates for the variational parameters and variance components using a scalable coordinate-ascent algorithm. 
    \item We theoretically establish the monotone convergence of ELBO in this scenario and justify the exactness of the Gaussian variational family.
\end{enumerate}
Beyond these advantages, we empirically explain the utility of variational REML (VRMLE) in terms of predictive performance compared to the MLE and INLA. In this paper, in Section~\ref{s:methods}, we explain our novel VRMLE in the ICAR setting for spatial data, and in Section~\ref{s:theory}, we provide theoretical results for the ELBO and the variational family. In Section~\ref{s:experiment}, we summarize our empirical results to compare the performance of our proposed VRMLE via a simulation study and breast cancer data.

\section{Methods}\label{s:methods}
Suppose we have the spatial observations for $n$ areal units in the domain of interest, $\mathcal S_n$. The spatially correlated response is defined as $\bm Y = (Y_1,\dots, Y_n)^\top \in \mathbb R^n$ and the design matrix $\bm X \in \mathbb R^{n\times p}$, where we assume that $p \ll n$. Let's consider, under the Gaussianity assumption, that the spatial response is generated as follows:
\begin{equation}\label{eq:data-gen}
\bm Y \mid \bm \beta,\bm u,\tau_y \sim \mathcal N\bigl(\bm X\bm \beta + \bm u,\; \tau_y^{-1}\bm I_n\bigr),
\end{equation}
where $\bm \beta \in \mathbb R^p$ is the regression coefficient vector, $\bm u\in\mathbb R^n$ is the latent spatial random effect, and $\tau_y>0$ is the precision parameter. For the spatial random effect, we consider \cite{besag1995conditional}'s intrinsic conditional autoregressive (ICAR) model before the random effect, $\bm u$ in \eqref{eq:data-gen}. Let $\bm W=(w_{ij})$ be the adjacency matrix such that $w_{ij}=1$ if areas $i$ and $j$ are neighbors and $w_{ij}=0$ otherwise and consider $d_i = \sum_{j=1}^n w_{ij}, \bm D = \mathrm{diag}(d_1,\dots,d_n), \bm R = \bm D - \bm W.$ Then the ICAR prior density can be formally written as:
\begin{equation}\label{eq:icar_prior}
g(\bm u \mid \tau_u)
\propto
\tau_u^{(n-r)/2}\exp\left(-\frac{\tau_u}{2}\bm u^\top \bm R \bm u\right),
\end{equation}
where $\tau_u>0$ is the spatial precision parameter and $r$ is the rank deficiency of matrix $\bm R$ mentioned in \eqref{eq:icar_prior}. For a connected graph, $r=1$, so that $\mathrm{rank}(\bm R)=n-1$. Since the ICAR prior is improper, an identifiability constraint, $\bm 1^\top \bm u = 0$, is typically imposed. The joint density of $(\bm Y,\bm u)$ conditional on the parameters, $(\bm \beta,\tau_y,\tau_u)$ is defined as:
\begin{equation}\label{eq:joint-density}
g(\bm Y,\bm u\mid \bm \beta,\tau_y,\tau_u)
=
g(\bm Y\mid \bm \beta,\bm u,\tau_y)\times g(\bm u\mid\tau_u).
\end{equation}
The joint estimates of $\bm \Theta = \left[\bm \beta, \tau_y, \tau_u\right]^\top$ by maximizing the joint likelihood from \eqref{eq:joint-density}, the estimate of the variance parameter can be biased; therefore, \cite{jiang1996reml} has proposed restricted maximum likelihood (REML) for the linear mixed effect model. Under the assumption of a flat prior for the regression coefficient, namely $g(\bm \beta)\propto 1$, the restricted log-likelihood is defined as
\begin{equation}\label{eq:restricted-likelihood}
\begin{split}
    L(\tau_y,\tau_u;\bm Y)
\propto
\log \int_{\mathbb R^p}
\int
g(\bm Y,\bm u\mid \bm \beta,\tau_y,\tau_u)\,
d\bm u\,d\bm \beta=\log 
\int
g(\bm Y,\bm u\mid\tau_y,\tau_u)\,
d\bm u.
\end{split}
\end{equation}
Although the restricted likelihood in \eqref{eq:restricted-likelihood} admits a closed-form representation on the constrained subspace, its evaluation still requires factorization or inversion of large, sparse precision matrices. For large \(n\), this becomes computationally expensive. This formulation is equivalent to the classical REML representation based on error contrasts, obtained by integrating out the fixed effects under a flat prior. By integrating out \(\bm\beta\) under a flat prior using standard Gaussian integration, the marginal density of \(\bm Y\) given \(\bm u\) and \(\tau_y\) becomes:
\begin{equation}\label{eq:marginal-y-given-u}
g(\bm Y\mid \bm u,\tau_y)
\propto
\tau_y^{(n-p)/2}
|\bm X^\top \bm X|^{-1/2}
\exp\left\{
-\frac{\tau_y}{2}
(\bm Y-\bm u)^\top \bm P_X^\perp (\bm Y-\bm u)
\right\},
\end{equation}
where $\bm P_X^\perp = \bm I_n - \bm X(\bm X^\top \bm X)^{-1}\bm X^\top$ is the orthogonal projection matrix onto the complement of the column space of the design matrix, $\bm X$, assuming the design matrix, $\bm X$, has full column rank. Replacing \eqref{eq:marginal-y-given-u} and \eqref{eq:icar_prior} in \eqref{eq:joint-density} the joint log-density after integrating out the regression coefficient, $\bm \beta$ becomes 
\begin{equation}\label{eq:restricted-joint-log}
    \begin{split}
        \widetilde{\mathcal L}(\bm u, \tau_y, \tau_u) = \log g(\bm Y,\bm u\mid \tau_y,\tau_u)
&=
\frac{n-p}{2}\log \tau_y
-\frac{1}{2}\log |\bm X^\top \bm X|
-\frac{\tau_y}{2}(\bm Y-\bm u)^\top \bm P_X^\perp (\bm Y-\bm u)\\
&\quad
+\frac{n-r}{2}\log \tau_u
-\frac{\tau_u}{2}\bm u^\top \bm R \bm u
+ \mathrm{const.}
    \end{split}
\end{equation}
Although the posterior density, $g(\bm u | \bm Y, \tau_y, \tau_u)$, is Gaussian in the constrained subspace, the computation of the mean and variance-covariance matrix is computationally extensive for large sample sizes; we introduce the variational density function $q(\bm u)$. A natural computationally scalable choice for the variational density function is $\mathcal{N}(\bm \mu, \bm \Sigma)$ where $\bm \mu \in \R^n$ and the variance-covariance matrix $\bm \Sigma \in \R^{n \times n}$ is a positive definite matrix on the constrained subspace. According to \cite{bishop2006pattern}, from Jensen's inequality, we have 
\begin{align}
L(\tau_y,\tau_u)
&=
\log \int q(\bm u)\,
\frac{g(\bm Y,\bm u\mid \tau_y,\tau_u)}{q(\bm u)}
\,d\bm u
\nonumber\\
&\ge
\int q(\bm u)\log \frac{g(\bm Y,\bm u\mid \tau_y,\tau_u)}{q(\bm u)}\,d\bm u = \mathcal L_V(q,\tau_y,\tau_u).
\label{eq:jensen-reml}
\end{align}
In this formulation, our main concern is to develop a variational approximation of the restricted maximum likelihood estimator instead of detailed variational Bayesian inference. Here, the latent spatial random effect, $\bm u$, is modeled via a Gaussian variational distribution; however, the variance components, $(\tau_y, \tau_u)$, are treated as deterministic and are estimated through the direct optimization of the evidence lower bound (ELBO). Therefore, we do not consider any prior distribution for $\tau_y$ and $\tau_u$, unlike the Bayesian hierarchical modeling of \cite{kang2011bayesian}. The lower-bound in \eqref{eq:jensen-reml} is achieved when the variational density function $q(\bm u)$ is equivalent to the posterior density. The variational log-likelihood (ELBO), derived from the log-likelihood from \eqref{eq:restricted-joint-log}, is defined as 
\begin{equation}\label{eq:elbo-reml}
\begin{split}
    \mathcal L_V(q,\tau_y,\tau_u)
&=
\mathbb E_q\!\left[\log g(\bm Y,\bm u\mid \tau_y,\tau_u)\right]
-
\mathbb E_q\!\left[\log q(\bm u)\right]\\
&= \frac{n-p}{2}\log \tau_y
-\frac{1}{2}\log |\bm X^\top \bm X|
-\frac{\tau_y}{2}\,
\mathbb E_q\!\left[
(\bm Y-\bm u)^\top \bm P_X^\perp (\bm Y-\bm u)
\right]\\
&\quad
+\frac{n-r}{2}\log \tau_u
-\frac{\tau_u}{2}\,
\mathbb E_q\!\left[\bm u^\top \bm R \bm u\right]
-
\mathbb E_q[\log q(\bm u)]
+ \mathrm{const.}
\end{split}
\end{equation}
Since $q(\bm u)=\mathcal N(\bm \mu,\bm \Sigma)$, where $\bm 1^\top \bm \mu = 0, \bm \Sigma \bm 1 = 0$. Then, using the standard identities, we finally have
\begin{equation}\label{eq:quad}
\begin{split}
    \mathbb E_q\!\left[
(\bm Y-\bm u)^\top \bm P_X^\perp (\bm Y-\bm u)
\right]
&=
(\bm Y-\bm \mu)^\top \bm P_X^\perp (\bm Y-\bm \mu)
+
\mathrm{tr}(\bm P_X^\perp \bm \Sigma),\\
\mathbb E_q\!\left[\bm u^\top \bm R \bm u\right]
&=
\bm \mu^\top \bm R\bm \mu
+
\mathrm{tr}(\bm R\bm \Sigma),\\
-\mathbb E_q[\log q(\bm u)]
&=
\frac{1}{2}\log |\bm \Sigma|_+
+\frac{n-r}{2}(1+\log 2\pi).
\end{split}
\end{equation}
Now, substituting \eqref{eq:quad} into \eqref{eq:elbo-reml}, and absorbing all terms constant in $(\bm\mu,\bm\Sigma,\tau_y,\tau_u)$ into the generic constant, the simplified ELBO objective function becomes
\begin{equation}\label{eq:final-elbo-reml}
    \begin{split}
      \mathcal L_V(\bm \mu,\bm \Sigma,\tau_y,\tau_u)
&=
\frac{n-p}{2}\log \tau_y
-\frac{\tau_y}{2}
\left[
(\bm Y-\bm \mu)^\top \bm P_X^\perp (\bm Y-\bm \mu)
+
\mathrm{tr}(\bm P_X^\perp \bm \Sigma)
\right]\\
&\quad
+\frac{n-r}{2}\log \tau_u
-\frac{\tau_u}{2}
\left[
\bm \mu^\top \bm R\bm \mu
+
\mathrm{tr}(\bm R\bm \Sigma)
\right]
+\frac{1}{2}\log |\bm \Sigma|_+
+ \mathrm{const.}
    \end{split}
\end{equation}
The variational REML estimators are defined by maximizing the ELBO objective function in \eqref{eq:final-elbo-reml} as follows:
\begin{equation}\label{eq:VRMLE-est}
(\widehat{\bm \mu},\widehat{\bm \Sigma},\widehat{\tau}_y,\widehat{\tau}_u)
=
\arg\max_{\bm \mu,\bm \Sigma,\tau_y,\tau_u}
\mathcal L_V(\bm \mu,\bm \Sigma,\tau_y,\tau_u).
\end{equation}
This optimization may be performed by coordinate ascent, alternating between updates for the variational parameters $(\bm \mu,\bm \Sigma)$ and the precision parameters $(\tau_y,\tau_u)$. To get optimal $(\widehat{\bm \mu}, \widehat{\bm \Sigma}, \widehat{\tau}_y, \widehat{\tau}_u)$ in \eqref{eq:VRMLE-est} we require to differentiate the variational likelihood in \eqref{eq:final-elbo-reml} w.r.t $(\bm \mu, \bm \Sigma, \tau_y, \tau_u)$ and as a result the set of equations are
\begin{equation}\label{eq:derivative}
\begin{split}
\frac{\partial \mathcal L_V}{\partial \bm \mu}
&=
\tau_y \bm P_X^\perp (\bm Y-\bm \mu)-\tau_u \bm R\bm \mu
=
\bm 0,\quad
\frac{\partial \mathcal L_V}{\partial \bm \Sigma}
=
-\frac{\tau_y}{2}\bm P_X^\perp
-\frac{\tau_u}{2}\bm R
+\frac{1}{2}\bm \Sigma_*^{-1}
=
\bm 0,\\
\frac{\partial \mathcal L_V}{\partial \tau_y}
&=
\frac{n-p}{2\tau_y}
-\frac{1}{2}
\left[
(\bm Y-\bm \mu)^\top \bm P_X^\perp (\bm Y-\bm \mu)
+
\mathrm{tr}(\bm P_X^\perp \bm \Sigma)
\right]
=
0,\,
\frac{\partial \mathcal L_V}{\partial \tau_u}
=
\frac{n-r}{2\tau_u}
-\frac{1}{2}
\left[
\bm \mu^\top \bm R\bm \mu
+
\mathrm{tr}(\bm R\bm \Sigma)
\right]
=
0.
\end{split}
\end{equation}
Here \(\bm\Sigma_*^{-1}\) denotes the inverse of \(\bm\Sigma\) on the constrained subspace
\(\mathcal E=\{\bm v\in\mathbb R^n:\bm 1^\top \bm v=0\}\), where $E_\ast^{-1}$ denotes the inverse of the matrix $E$ over the constrained subspace $\mathcal E = \left\{\bm v\in \R^n: \bm 1^\top \bm v = \bm 0\right\}$. Now, by setting the first-order stationarity conditions of \eqref{eq:derivative} to zero, we obtain the following fixed-point equations:
\begin{equation}\label{eq:optimal-solution}
    \begin{split}
    \widehat{\bm \mu} &= \bigl(\tau_y \bm P_X^\perp + \tau_u \bm R\bigr)_\ast^{-1} \tau_y \bm P_X^\perp \bm Y,\quad
    \widehat{\bm \Sigma} = \bigl(\tau_y \bm P_X^\perp + \tau_u \bm R\bigr)_\ast^{-1},\\
    \widehat{\tau}_y &= \frac{n-p}{
(\bm Y-\bm \mu)^\top \bm P_X^\perp (\bm Y-\bm \mu) + \mathrm{tr}(\bm P_X^\perp \bm \Sigma)},\quad
\widehat{\tau}_u
= \frac{n-r}{ \bm \mu^\top \bm R\bm \mu + \mathrm{tr}(\bm R\bm \Sigma)}.
    \end{split}
\end{equation}
Since \(\bm P_X^\perp\) and \(\bm R\) may both be singular on \(\mathbb R^n\), the inverse of \(\tau_y \bm P_X^\perp+\tau_u \bm R\) in \eqref{eq:optimal-solution} is understood as the inverse restricted to the constrained subspace \(\mathcal E\). Using the optimal solutions from \eqref{eq:optimal-solution}, the practical algorithm for maximizing the restricted variational lower bound proceeds as follows.

\medskip

\noindent
\textbf{Algorithm: Variational REML for the ICAR model}

\begin{enumerate}
    \item Initialize $\tau_y^{(0)}>0$ and $\tau_u^{(0)}>0$.
    \item For iteration $t=0,1,2,\dots$:
    \begin{enumerate}
        \item Update the variational covariance
        \[
        \bm \Sigma^{(t+1)}
        =
        \bigl(\tau_y^{(t)}\bm P_X^\perp + \tau_u^{(t)}\bm R\bigr)_\ast^{-1}.
        \]
        \item Update the variational mean
        \[
        \bm \mu^{(t+1)}
        =
        \bigl(\tau_y^{(t)}\bm P_X^\perp + \tau_u^{(t)}\bm R\bigr)_\ast^{-1}
        \tau_y^{(t)}\bm P_X^\perp \bm Y.
        \]
        \item Update the observation precision
        \[
        \tau_y^{(t+1)}
        =
        \frac{n-p}{
        (\bm Y-\bm \mu^{(t+1)})^\top \bm P_X^\perp (\bm Y-\bm \mu^{(t+1)})
        +
        \mathrm{tr}(\bm P_X^\perp \bm \Sigma^{(t+1)})
        }.
        \]
        \item Update the spatial precision
        \[
        \tau_u^{(t+1)}
        =
        \frac{n-r}{
        (\bm \mu^{(t+1)})^\top \bm R\bm \mu^{(t+1)}
        +
        \mathrm{tr}(\bm R\bm \Sigma^{(t+1)})
        }.
        \]
    \end{enumerate}
    \item Repeat until
    \[
    \left|
    \mathcal L_V^{(t+1)}-\mathcal L_V^{(t)}
    \right|
    < \epsilon
    \]
    for some small tolerance $\epsilon>0$.
\end{enumerate}
In the previous literature, the existing hybrid variational methods, for example \cite{loaiza2022fast}, approximate Bayesian posteriors for the latent variable models by integrating a variational approximation for the global parameters via simulating from the conditional posterior distribution of latent variables. But the main limitation in this approach is that it requires repeated Monte Carlo sampling during its optimization; however, \cite{loaiza2022fast} do not exploit the sparse graphical feature of ICAR models. In contrast, our proposed VREML framework performs deterministic variational optimization directly on the REML objective, which leverages the sparse precision structure of the ICAR model, and accommodates singular precision matrices, avoiding repeated simulation from the latent field.

But one disadvantage in our proposed VRMLE is that the repeated computation of the $n \times n$ variance-covariance matrix, $\bm \Sigma$, causes large computational complexity for a large number of spatial units, although the original ICAR precision matrix is sparse; after inversion, the variational covariance matrix becomes denser. Therefore, to mitigate computational burden, we follow \cite{ong2018gaussian}'s suggested low-rank variance-covariance matrix approximation, and the variance-covariance matrix can be decomposed as $\bm\Sigma = \bm B\bm B^\top+\bm D^2$, where $\bm B\in\mathbb R^{n\times k}$ is a low-rank basis matrix with $k \ll n$, and the diagonal matrix $\bm D=\mathrm{diag}(d_1,\ldots,d_n)$. Here, the main advantage is that the computational complexity has been reduced from $O(n^2)$ to $O(n\cdot k+n)$. Using this reparametrization, substituting \eqref{eq:quad} into \eqref{eq:elbo-reml}, and absorbing all terms constant in \((\bm\mu,\bm\Sigma,\tau_y,\tau_u)\) into the generic constant, the simplified ELBO objective function becomes
\begin{equation}\label{eq:reparm-elbo-reml}
    \begin{split}
        \mathcal L_V(\bm \mu,\bm \Sigma,\tau_y,\tau_u)
&=
\frac{n-p}{2}\log \tau_y
-\frac{\tau_y}{2}
\left[
(\bm Y-\bm \mu)^\top \bm P_X^\perp (\bm Y-\bm \mu)
+
\mathrm{tr}\left[\bm P_X^\perp \left(\bm B\bm B^\top+\bm D^2\right)\right]
\right]\\
&\quad
+\frac{n-r}{2}\log \tau_u
-\frac{\tau_u}{2}
\left[
\bm \mu^\top \bm R\bm \mu
+
\mathrm{tr}\left[\bm R\cdot \left(\bm B\bm B^\top+\bm D^2\right)\right]
\right]
+\frac{1}{2}\log |\left(\bm B\bm B^\top+\bm D^2\right)|_+
+ \mathrm{const.}
    \end{split}
\end{equation}
The approximated VRMLE estimators can be obtained by optimizing the modified ELBO objective function from \eqref{eq:reparm-elbo-reml} with respect to $(\bm \mu, \bm B, \bm D, \tau_y, \tau_u)$. The mean update remains unchanged, but for $\tau_y$ and $\tau_u$ estimate the variance-covariance matrix $\bm \Sigma$ will be replaced by $\bm B\bm B^\top + \bm D^2$ in $(\widehat{\tau}_y, \widehat{\tau}_u)$ in \eqref{eq:optimal-solution}. Now for fixed $(\bm \mu, \tau_y, \tau_u)$ the covariance parameters, $(\bm B, \bm D)$ can be updated considering $\bm d = (d_1, \dots, d_n)$ and $\bm D = \mathrm{diag}(\bm d)$ by optimizing the following ELBO 
\begin{equation}\label{eq:approx_elbo_grad}
    \begin{split}
   \mathcal L_C(\bm B,\bm d) = -\frac12 \mathrm{tr} \left(\left[\tau_y\bm P_X^\perp+\tau_u\bm R\right]
\left[\bm B \bm B^\top + \bm D^2 \right]\right) + \frac12 \log\left|\left[\bm B \bm B^\top + \bm D^2 \right]\right|_+\\
\nabla_{\bm B}\mathcal L_C
=
\left(\left[\bm B \bm B^\top + \bm D^2 \right]^{-1}-\left[\tau_y\bm P_X^\perp+\tau_u\bm R\right]
\right)\bm B, \\
\nabla_{\bm d}\mathcal L_C = \mathrm{diag} \left(\left[\bm B \bm B^\top + \bm D^2 \right]^{-1}-\left[\tau_y\bm P_X^\perp+\tau_u\bm R\right]\right)\odot \bm d.
    \end{split}
\end{equation}
Using the matrix differential identities, we have established the gradients of the reparametrized ELBO function in \eqref{eq:approx_elbo_grad}, where $\odot$ indicates componentwise multiplication. Therefore, in the gradient-ascent step, we can update as follows
\begin{equation*}
    \begin{split}
\bm B^{(t+1)} = \bm B^{(t)} + \rho_t \left[ \left( (\bm\Sigma_k^{(t)})^{-1}-\bm C \right) \bm B^{(t)}\right] \\
\bm d^{(t+1)} = \bm d^{(t)} + \rho_t \left[\mathrm{diag}
\left((\bm\Sigma_k^{(t)})^{-1}-\bm C\right)\odot\bm d^{(t)}
\right]
    \end{split}
\end{equation*}
assuming $\bm \Sigma_k = \bm B \bm  B^\top + \bm D^2$ and $\bm C = \tau_y \bm P_X^\perp + \tau_u \bm R$ and $\rho_t$ as a step size corresponding to the $t$th iteration. Since using Woodbury's matrix inversion identity, we have the following expansion,
\[
(\bm B\bm B^\top+\bm D^2)^{-1} = \bm D^{-2} - \bm D^{-2}\bm B
\left( \bm I_k+\bm B^\top\bm D^{-2}\bm B \right)^{-1} \bm B^\top\bm D^{-2}.
\]
and $\bm D^2 = \mathrm{diag}(d_1^2, \dots, d_n^2)$ is a diagonal matrix and $\bm B \in \mathbb R^{n \times k}$ for $k \ll n$ we only require the inversion of $k \times k$ dimensional matrix in place of the inversion of $n \times n$ matrix, which reduces the computational cost extensively. This approach can be considered as a low-rank covariance approximated VRMLE (LRCVRMLE). However, under the Gaussian ICAR model assumption, the unrestricted Gaussian variational family already yields the exact posterior distribution; we primarily focus on the unrestricted covariance formulation throughout the remainder of the article.

\section{Theoretical Results}\label{s:theory}
In this section, we discuss different theoretical properties of our proposed variational restricted maximum likelihood (VRMLE) procedure for the Gaussian ICAR model. Let's consider the constrained subspace $\mathcal E=\left\{ \bm v\in\mathbb R^n:\bm 1_n^\top \bm v=0 \right\}$. Likewise, in Section~\ref{s:methods}, consider the ELBO loss function from \eqref{eq:elbo-reml}. Here \(\bm \mu\in \mathcal E\), \(\bm \Sigma\succ 0\) on \(\mathcal E\), \(\tau_y>0,\tau_u>0\), and \(|\bm\Sigma|_+\) denotes the pseudo-determinant of \(\bm\Sigma\) on the constrained subspace $\mathcal E$. Let's assume the ELBO at iteration \(t\), during the coordinate-ascent algorithm update is, $\mathcal L_V^{(t)} = \mathcal L_V\bigl(\bm\mu^{(t)},\bm\Sigma^{(t)},\tau_y^{(t)},\tau_u^{(t)}\bigr)$ and consider $\bigl(\bm\mu^{(t)},\bm\Sigma^{(t)},\tau_y^{(t)},\tau_u^{(t)}\bigr)$ are the variational updates corresponding to the $t$th iteration where $t=0,1,2,\dots$. We consider regularity conditions in the following:
\begin{ass}\label{ass:design}
    The design matrix \(\bm X\) has full column rank and $p \ll n$.
\end{ass}
\begin{ass}\label{ass:adjacency}
    The adjacency graph is connected, so that \(\mathrm{rank}(\bm R)=n-r\) with \(r=1\) and $r < n$.
\end{ass}
\begin{ass}\label{ass:const-pd}
    The matrix \(\tau_y \bm P_X^\perp + \tau_u \bm R\) is positive definite on \(\mathcal E\).
\end{ass}
Assumption~\textbf{A}-\ref{ass:design} is a standard assumption about the design matrix which ensures that the columns of that design matrix are linearly independent, such that $\bm X^\top \bm X$ is invertible. The condition $p \ll n$ ensures that the model is defined in a low-dimensional regime, which makes the estimation stable. Assumption~\textbf{A}-\ref{ass:adjacency} guarantees that the underlying adjacency graph is connected. Assumption~\textbf{A}-\ref{ass:const-pd} eliminates degeneracy in the covariance structure, and this guarantees identifiability and plays a crucial role in establishing the strict concavity of the ELBO with respect to $\bm \Sigma$. Under these assumptions, the following theorem holds.
\begin{theorem}\label{thm:blockwise-concavity-monotone}
Under the regularity conditions from (\textbf{A}-\ref{ass:design})--(\textbf{A}-\ref{ass:const-pd}), each coordinate update in the variational REML algorithm is uniquely defined. Moreover, the sequence \(\{\mathcal L_V^{(t)}\}\) generated by the coordinate-ascent algorithm is monotone nondecreasing and converges to a finite limit. Any accumulation point of the sequence \(\{(\bm\mu^{(t)},\bm\Sigma^{(t)},\tau_y^{(t)},\tau_u^{(t)})\}\) is a stationary point of the ELBO in \eqref{eq:final-elbo-reml}.
\end{theorem}

\begin{pf} At iteration $t$ for the coordinate ascent algorithm we have
\begin{equation}\label{eq:monotonicity}
    \begin{split}
 &\bm\Sigma^{(t+1)}
=
\arg\max_{\bm\Sigma\succ 0 \text{ on } \mathcal E}
\mathcal L_V\bigl(\bm\mu^{(t)},\bm\Sigma,\tau_y^{(t)},\tau_u^{(t)}\bigr) \implies \mathcal L_V\bigl(\bm\mu^{(t)},\bm\Sigma^{(t+1)},\tau_y^{(t)},\tau_u^{(t)}\bigr)
\ge
\mathcal L_V\bigl(\bm\mu^{(t)},\bm\Sigma^{(t)},\tau_y^{(t)},\tau_u^{(t)}\bigr),\\
&\bm\mu^{(t+1)}
=
\arg\max_{\bm\mu\in \mathcal E}
\mathcal L_V\bigl(\bm\mu,\bm\Sigma^{(t+1)},\tau_y^{(t)},\tau_u^{(t)}\bigr)\implies \mathcal L_V\bigl(\bm\mu^{(t+1)},\bm\Sigma^{(t+1)},\tau_y^{(t)},\tau_u^{(t)}\bigr)
\ge
\mathcal L_V\bigl(\bm\mu^{(t)},\bm\Sigma^{(t+1)},\tau_y^{(t)},\tau_u^{(t)}\bigr),\\
&\tau_y^{(t+1)}
=
\arg\max_{\tau_y>0}
\mathcal L_V\bigl(\bm\mu^{(t+1)},\bm\Sigma^{(t+1)},\tau_y,\tau_u^{(t)}\bigr) \implies \mathcal L_V\bigl(\bm\mu^{(t+1)},\bm\Sigma^{(t+1)},\tau_y^{(t+1)},\tau_u^{(t)}\bigr)
\ge
\mathcal L_V\bigl(\bm\mu^{(t+1)},\bm\Sigma^{(t+1)},\tau_y^{(t)},\tau_u^{(t)}\bigr),\\
&\tau_u^{(t+1)}
=
\arg\max_{\tau_u>0}
\mathcal L_V\bigl(\bm\mu^{(t+1)},\bm\Sigma^{(t+1)},\tau_y^{(t+1)},\tau_u\bigr) \implies \mathcal L_V\bigl(\bm\mu^{(t+1)},\bm\Sigma^{(t+1)},\tau_y^{(t+1)},\tau_u^{(t+1)}\bigr)
\ge
\mathcal L_V\bigl(\bm\mu^{(t+1)},\bm\Sigma^{(t+1)},\tau_y^{(t+1)},\tau_u^{(t)}\bigr).
    \end{split}
\end{equation}
Thus, from the above inequalities in \eqref{eq:monotonicity} we have 
\begin{equation*}
    \begin{split}
\mathcal L_V\bigl(\bm\mu^{(t+1)},\bm\Sigma^{(t+1)},\tau_y^{(t+1)},\tau_u^{(t+1)}\bigr)
\ge
\mathcal L_V\bigl(\bm\mu^{(t+1)},\bm\Sigma^{(t+1)},\tau_y^{(t+1)},\tau_u^{(t)}\bigr)
\ge\mathcal L_V\bigl(\bm\mu^{(t+1)},\bm\Sigma^{(t+1)},\tau_y^{(t)},\tau_u^{(t)}\bigr)\\
\ge
\mathcal L_V\bigl(\bm\mu^{(t)},\bm\Sigma^{(t+1)},\tau_y^{(t)},\tau_u^{(t)}\bigr)
\ge
\mathcal L_V\bigl(\bm\mu^{(t)},\bm\Sigma^{(t)},\tau_y^{(t)},\tau_u^{(t)}\bigr)
    \end{split}
\end{equation*}
which establishes the monotonicity i.e, $\mathcal L_V^{(t+1)}\ge \mathcal L_V^{(t)}$. Actually this monotonicity property of coordinate ascent algorithm is the ascent analogue of block coordinate descent algorithm of \cite{tseng2001convergence} which can be established by considering $\tilde{f}(\bm \Theta) = -\mathcal{L}_V(\bm \mu, \bm \Sigma, \tau_y, \tau_u)$ where $\bm \Theta = (\bm \mu, \bm \Sigma, \tau_y, \tau_u)$. Because with each block we perform an exact conditional minimization over the feasible parameter space and thus $\tilde{f}(\bm \Theta^{(t+1)}) \le \tilde{f}(\bm \Theta^{(t)})$ which justifies the equivalent monotonicity property of coordinate ascent, $\mathcal L_V^{(t+1)}\ge \mathcal L_V^{(t)}$. From Jensen's inequality in \eqref{eq:jensen-reml}, we know that the ELBO is a lower bound to the restricted log-likelihood, i.e., $\mathcal L_V(q,\tau_y,\tau_u)\le L(\tau_y,\tau_u;\bm Y),$ and thus \(\mathcal L_V\) is bounded above on the parameter space; the sequence \(\{\mathcal L_V^{(t)}\}\) is monotone increasing and bounded above. Hence, by the monotone convergence theorem for real sequences, there exists a finite limit \(\overline{L}^\star\) such that $\mathcal L_V^{(t)} \to \overline{L}^\star$ whenever $t\to\infty.$ To study the behavior of the coordinate updates, we examine the concavity properties of the ELBO with respect to each block of parameters. For fixed \((\bm\Sigma,\tau_y,\tau_u)\), the terms in \eqref{eq:final-elbo-reml} depending on \(\bm\mu\) are $-\frac{\tau_y}{2}(\bm Y-\bm\mu)^\top \bm P_X^\perp (\bm Y-\bm\mu)
-\frac{\tau_u}{2}\bm\mu^\top \bm R \bm\mu.$ Expanding the first quadratic form we get
% \begin{equation}\label{eq:hessian_mu}
%     \begin{split}
%       (\bm Y-\bm\mu)^\top \bm P_X^\perp (\bm Y-\bm\mu)
% &=
% \bm Y^\top \bm P_X^\perp \bm Y
% -2\bm\mu^\top \bm P_X^\perp \bm Y
% +\bm\mu^\top \bm P_X^\perp \bm\mu\\
% &=\bm\mu^\top (\tau_y \bm P_X^\perp \bm Y)
% -\frac{1}{2}\bm\mu^\top (\tau_y \bm P_X^\perp+\tau_u \bm R)\bm\mu
% +\mathrm{const.}\\
% \implies \nabla_{\bm\mu\bm\mu}^2 \mathcal L_V
% &=
% -(\tau_y \bm P_X^\perp+\tau_u \bm R).
%     \end{split}
% \end{equation}
\begin{equation}\label{eq:hessian_mu}
\begin{split}
(\bm Y-\bm\mu)^\top \bm P_X^\perp (\bm Y-\bm\mu)
&=
\bm Y^\top \bm P_X^\perp \bm Y
-2\bm\mu^\top \bm P_X^\perp \bm Y
+\bm\mu^\top \bm P_X^\perp \bm\mu,\\
\implies \nabla_{\bm\mu\bm\mu}^2 \mathcal L_V
&=
-(\tau_y \bm P_X^\perp+\tau_u \bm R).
\end{split}
\end{equation}
Since \(\bm P_X^\perp\succeq 0\) and \(\bm R\succeq 0\), from \eqref{eq:hessian_mu} it follows that $\tau_y \bm P_X^\perp+\tau_u \bm R \succeq 0,$ and therefore $\nabla_{\bm\mu\bm\mu}^2 \mathcal L_V \preceq 0.$ Thus, the ELBO is concave in \(\bm\mu\). Under the condition (\textbf{A}-\ref{ass:const-pd}), the matrix \(\tau_y \bm P_X^\perp+\tau_u \bm R\) is positive definite on \(\mathcal E\), so the Hessian is negative definite on \(\mathcal E\). Hence the ELBO is strictly concave in \(\bm\mu\) on \(\mathcal E\). For fixed \((\bm\mu,\tau_y,\tau_u)\), the \(\bm\Sigma\)-dependent part of the ELBO is $-\frac{\tau_y}{2}\,\mathrm{tr}(\bm P_X^\perp\bm\Sigma)
-\frac{\tau_u}{2}\,\mathrm{tr}(\bm R\bm\Sigma)
+\frac{1}{2}\log |\bm\Sigma|_+.$ The first two terms are linear in \(\bm\Sigma\), hence those are affine functions of $\bm \Sigma$ and do not contribute to the curvature of the objective function. The map $\bm\Sigma \mapsto \log |\bm\Sigma|_+$ is strictly concave on the cone of positive definite matrices over \(\mathcal E\). Therefore, the entire expression is strictly concave in \(\bm\Sigma\). There are no cross product of \(\bm\mu\) and \(\bm\Sigma\) terms in \eqref{eq:final-elbo-reml}. Hence, the ELBO is the sum of a concave function of \(\bm\mu\) and a concave function of \(\bm\Sigma\). Since there are no cross terms between \(\bm\mu\) and \(\bm\Sigma\), the ELBO is separable in these parameters and hence jointly concave in \((\bm\mu,\bm\Sigma)\). For fixed \((\bm\mu,\bm\Sigma,\tau_u)\), define $T_1(\bm\mu,\bm\Sigma)
=
(\bm Y-\bm\mu)^\top \bm P_X^\perp (\bm Y-\bm\mu)
+
\mathrm{tr}(\bm P_X^\perp \bm\Sigma).$ Because \(\bm P_X^\perp\succeq 0\) and \(\bm\Sigma\succeq 0\) on \(\mathcal E\), we have \(T_1(\bm\mu,\bm\Sigma)\ge 0\). The \(\tau_y\)-dependent part of the ELBO is
\begin{equation}\label{eq:tau_y}
    h(\tau_y)
=
\frac{n-p}{2}\log \tau_y
-\frac{\tau_y}{2}T_1(\bm\mu,\bm\Sigma),\; h'(\tau_y)
=
\frac{n-p}{2\tau_y}
-\frac{1}{2}T_1(\bm\mu,\bm\Sigma),
\;
h''(\tau_y)
=
-\frac{n-p}{2\tau_y^2}<0
\quad (\tau_y>0).
\end{equation}
Thus, from \eqref{eq:tau_y} the measurable function, \(h\) is strictly concave in \(\tau_y\). Similarly, define $B(\bm\mu,\bm\Sigma)
= \bm\mu^\top \bm R \bm\mu + \mathrm{tr}(\bm R \bm\Sigma).$ Since \(\bm R\succeq 0\) and \(\bm\Sigma\succeq 0\) on \(\mathcal E\), we have \(B(\bm\mu,\bm\Sigma)\ge 0\). The \(\tau_u\)-dependent part of the ELBO is
\begin{equation}\label{eq:tau_u}
    f(\tau_u)
=
\frac{n-r}{2}\log \tau_u
-\frac{\tau_u}{2}B(\bm\mu,\bm\Sigma), \; f'(\tau_u)
= \frac{n-r}{2\tau_u}
-\frac{1}{2}B(\bm\mu,\bm\Sigma),
\; f''(\tau_u) = -\frac{n-r}{2\tau_u^2}<0
\quad (\tau_u>0).
\end{equation}
Therefore, from \eqref{eq:tau_u}, the function \(f\) is strictly concave in \(\tau_u\). Hence, each block update has a unique maximizer. Combined with the monotonicity established in \eqref{eq:monotonicity} and the upper boundedness of the ELBO, this proves convergence of the ELBO sequence to a finite limit. Standard coordinate-ascent arguments then imply that any accumulation point is a stationary point of the ELBO.
\end{pf}

\begin{rmk}
Theorem~\ref{thm:blockwise-concavity-monotone} provides the main algorithmic justification for the proposed variational REML estimator. In particular, each coordinate update is well-defined and unique, the ELBO cannot decrease across iterations, and the sequence of ELBO values converges to a finite limit. The theorem guarantees convergence to a stationary point of the ELBO, rather than asserting global optimality.
\end{rmk}
Next, we show that for the Gaussian ICAR model, the Gaussian variational family on the ICAR constrained subspace is exact. Consequently, the evidence lower bound (ELBO) attains the restricted log-likelihood at the true posterior distribution. Under the spatial data generation model from \eqref{eq:data-gen} and the ICAR prior from \eqref{eq:icar_prior} let's consider the variational family as 
\begin{equation}\label{eq:variation-family}
   \mathcal Q = \left\{ \mathcal N_{\mathcal E}(\bm\mu,\bm\Sigma):
\bm\mu\in\mathcal E,\; \bm\Sigma\succ 0 \text{ on }\mathcal E
\right\}. 
\end{equation}

\begin{theorem}\label{thm:exact-gaussian-variational-family}
Under the spatial generation model from \eqref{eq:data-gen}, ICAR prior from \eqref{eq:icar_prior}, the variational family from \eqref{eq:variation-family} then under the the regularity conditions from (\textbf{A}-\ref{ass:design})-(\textbf{A}-\ref{ass:const-pd}), we have
\begin{equation}\label{eq:var-opt}
    \sup_{q\in\mathcal Q}\mathcal L_V(q,\tau_y,\tau_u)
=
L(\tau_y,\tau_u;\bm Y).
\end{equation}
\end{theorem}

\begin{pf} After integrating out \(\bm\beta\), the marginal density of \(\bm Y\) given \(\bm u\) and \(\tau_y\) from \eqref{eq:marginal-y-given-u} and combining this with the ICAR prior from \eqref{eq:icar_prior} up to a normalizing constant we have
\begin{equation}\label{eq:post_u_y}
    g(\bm u\mid \bm Y,\tau_y,\tau_u)
\propto
\exp\!\left\{
-\frac{\tau_y}{2}
(\bm Y-\bm u)^\top \bm P_X^\perp (\bm Y-\bm u)
-\frac{\tau_u}{2}\bm u^\top \bm R \bm u
\right\},
\qquad \bm u\in\mathcal E.
\end{equation}
Expand the first quadratic term:
\begin{equation}\label{eq:quad_pf}
    (\bm Y-\bm u)^\top \bm P_X^\perp (\bm Y-\bm u)
=
\bm Y^\top \bm P_X^\perp \bm Y
-2\bm u^\top \bm P_X^\perp \bm Y
+\bm u^\top \bm P_X^\perp \bm u.
\end{equation}
Substituting the expansion from \eqref{eq:quad_pf} into the exponent of \eqref{eq:post_u_y}, the first term does not depend on \(\bm u\). Therefore, the posterior kernel can be written as:
\[
\begin{aligned}
-\frac{\tau_y}{2}
(\bm Y-\bm u)^\top \bm P_X^\perp (\bm Y-\bm u)
-\frac{\tau_u}{2}\bm u^\top \bm R \bm u
&=
-\frac{\tau_y}{2}\bm Y^\top \bm P_X^\perp \bm Y
+\tau_y \bm u^\top \bm P_X^\perp \bm Y
-\frac{1}{2}\bm u^\top (\tau_y\bm P_X^\perp+\tau_u\bm R)\bm u,\\
\implies g(\bm u\mid \bm Y,\tau_y,\tau_u)
&\propto
\exp\!\left\{
-\frac{1}{2}\bm u^\top \bm A \bm u
+
\bm b^\top \bm u
\right\},
\qquad \bm u\in\mathcal E,
\end{aligned}
\]
where $\bm A=\tau_y\bm P_X^\perp+\tau_u\bm R, \bm b=\tau_y\bm P_X^\perp \bm Y.$ Since \(\bm A\) is positive definite on \(\mathcal E\), define $\bm\Sigma^\star=\bm A^{-1}_{\mathcal E},$ and $\bm\mu^\star=\bm A^{-1}_{\mathcal E}\bm b.$ Then we have
\[
\begin{aligned}
    &-\frac{1}{2}\bm u^\top \bm A \bm u+\bm b^\top \bm u
=
-\frac{1}{2}(\bm u-\bm\mu^\star)^\top \bm A(\bm u-\bm\mu^\star)
+\frac{1}{2}(\bm\mu^\star)^\top \bm A \bm\mu^\star,\\
\implies& g(\bm u\mid \bm Y,\tau_y,\tau_u)
\propto
\exp\!\left\{
-\frac{1}{2}
(\bm u-\bm\mu^\star)^\top
(\bm\Sigma^\star)^{-1}
(\bm u-\bm\mu^\star)
\right\},
\qquad \bm u\in\mathcal E,
\end{aligned}
\]
which is precisely the density of a Gaussian distribution on the constrained subspace, i.e. $g(\bm u\mid \bm Y,\tau_y,\tau_u) = \mathcal 526841N_{\mathcal E}(\bm\mu^\star,\bm\Sigma^\star).$ Thus the conditional posterior is Gaussian on \(\mathcal E\). Now let $q(\bm u)=N_{\mathcal E}(\bm\mu,\bm\Sigma)\in\mathcal Q$ and using formal variational identity from \cite{bishop2006pattern},
\[
L(\tau_y,\tau_u;\bm Y)
=
\mathcal L_V(q,\tau_y,\tau_u)
+
\mathrm{KL}\!\left(
q(\bm u)\,\|\,g(\bm u\mid \bm Y,\tau_y,\tau_u)
\right).
\]
Since the Kullback-Leibler divergence is always nonnegative so we have $\mathcal L_V(q,\tau_y,\tau_u)\le L(\tau_y,\tau_u;\bm Y).$ Thus the true posterior itself belongs to \(\mathcal Q\), namely
\[
g(\bm u\mid \bm Y,\tau_y,\tau_u)
=
\mathcal N_{\mathcal E}(\bm\mu^\star,\bm\Sigma^\star)\in\mathcal Q.
\]
Therefore, choosing $q^\star(\bm u)=g(\bm u\mid \bm Y,\tau_y,\tau_u)$ yields $\mathrm{KL}\!\left(
q^\star(\bm u)\,\|\,g(\bm u\mid \bm Y,\tau_y,\tau_u)
\right)=0,$ and hence $\mathcal L_V(q^\star,\tau_y,\tau_u)=L(\tau_y,\tau_u;\bm Y).$ Thus the Gaussian variational family is exact, and $\sup_{q\in\mathcal Q}\mathcal L_V(q,\tau_y,\tau_u) = L(\tau_y,\tau_u;\bm Y).$ This completes the proof.
\end{pf}

\begin{rmk}
Theorem~\ref{thm:exact-gaussian-variational-family} shows that for the Gaussian ICAR model, the Gaussian variational approximation does not introduce any approximation error at the posterior level. In this setting, the ELBO is tight and attains the restricted log-likelihood at the true constrained Gaussian posterior.
\end{rmk}

\section{Results and Discussions}\label{s:experiment}
In this section, we illustrate the performance of VRMLE with other competitive methods via a simulation study and a real data analysis. We conducted a detailed simulation study to compare the performance of two other well-known methods: (i) the integrated nested Laplace approximation (INLA), (ii) the traditional exact maximum likelihood estimator (MLE), with our newly proposed variational restricted maximum likelihood estimator (VRMLE) and LRCVRMLE. In this section, we thoroughly compare the performance of our VRMLE methods with other competitive approaches with the increasing spatial sample size, which is the cardinality of the spatial sampling domain, $\mathcal{S}_n$, increases. In a simulation study, we generate regularly spaced $n_0\times n_0$ areal data; as a result, the sample size $n = n_0^2$. In this setup, we have fixed $p=3$ and generated the response like \eqref{eq:data-gen} and consider the design matrix as: 
\[
\bm X = \begin{bmatrix}
1 & x_{11} & x_{21}\\
\vdots & \vdots & \vdots\\
1 & x_{1n} & x_{2n}
\end{bmatrix},
\]
where $x_1$ and $x_2$ are the standardized row and column coordinates of the lattice locations, respectively; the true regression coefficient is $\bm \beta = (1.0,\ 1.2,\ -1.0)^\top$ and the true noise, $\varepsilon \sim \mathcal N_n(0, \sigma_\varepsilon^2 I_n)$ where we set $\sigma_\varepsilon^2 = 0.7$. Let's consider $\bm H\in \R^{n \times (n-1)}$ consists of orthonormal eigenvectors corresponding to the non-zero eigenvalues of, $\bm C_n = \bm I_n - \frac{1}{n}\bm 1 \bm 1^\top$ and it satisfies $\bm H^\top \bm H = \bm I_{n-1}$ and $\bm H^\top \bm 1 = 0$. Then for the spatial random effect, we consider $\bm u = \bm H\cdot \bm \Theta$ where we assume, $\bm \Theta \sim \mathcal N_{n-1}(\bm 0, \sigma_u^2 K^{-1})$ where $\bm K = \bm H^\top \bm R \bm H$ and as a result $\bm u \sim \mathcal{N}_n (\bm 0, \sigma_u^2 \bm H\bm K^{-1} \bm H^\top )$ by fixing $\sigma_u^2 = 1.3$. We run this simulation study $N_{\text{sim}}= 1000$ and compare average mean-square prediction error (mean MSPE), average mean square error (MSE) for the posterior mean and variance of the random effect, i.e., $\hat{\sigma}_u^2$, and for the error variance, we also consider average MSE which is considered as $\hat{\sigma}_\varepsilon^2$. 
\begin{figure}
    \centering
    \includegraphics[width=0.5\linewidth]{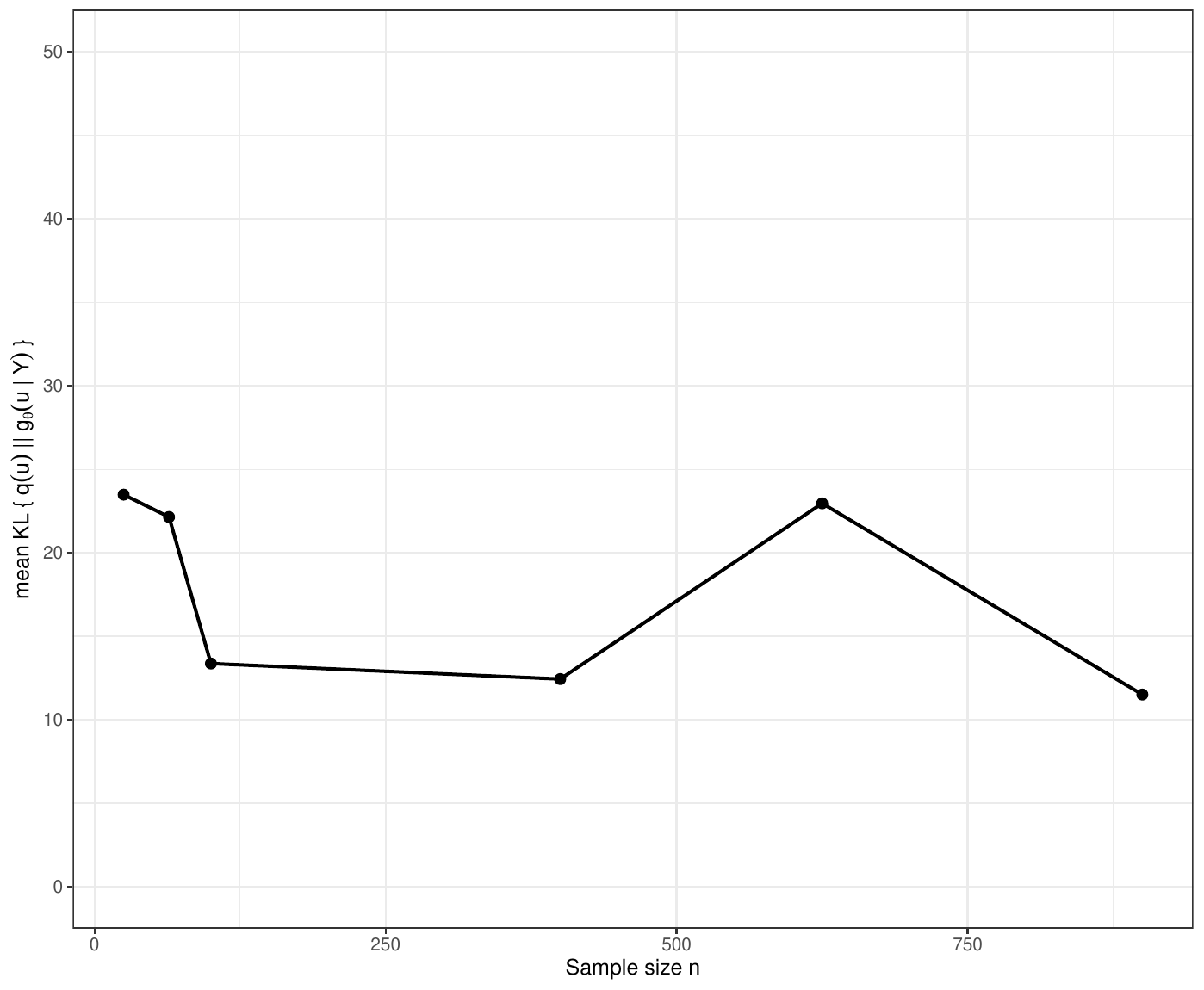}
    \caption{Mean Kullback–Leibler divergence between the variational posterior, $q(\bm u)$ and the true posterior, $g_{\bm \theta}(\bm u |\bm Y)$ across different sample sizes.}
    \label{fig:mean_kl}
\end{figure}
\begin{table}[ht]
\centering
\caption{Simulation performance comparison across methods and sample sizes.}
\label{tab:metrics_simulation}
\begin{tabular}{llcccc}
\hline
Method & $n$ & Avg. MSPE & Avg. Posterior $\bm u$ Mean MSE & Avg. Posterior $\bm u$ Variance MSE & Avg. $\bm \beta$ MSE \\
\hline
INLA & 25  & 0.989 & 0.468 & 0.447 & 0.414 \\
INLA & 64  & 1.522 & 0.708 & 0.482 & 0.183 \\
INLA & 100 & 1.283 & 0.846 & 0.498 & 0.096 \\
INLA & 400 & 1.322 & 0.437 & 0.049 & 0.029 \\
INLA & 625 & 1.139 & 0.527 & 0.001 & 0.030 \\
INLA & 900 & 1.306 & 0.587 & 0.038 & 0.051 \\
\hline
LRCVRMLE & 25  & 0.998 & 0.442 & 0.335 & 0.424 \\
LRCVRMLE & 64  & 1.624 & 0.684 & 0.242 & 0.227 \\
LRCVRMLE & 100 & 1.487 & 0.995 & 0.201 & 0.083 \\
LRCVRMLE & 400 & 1.718 & 0.922 & 0.240 & 0.060 \\
LRCVRMLE & 625 & 1.744 & 1.119 & 0.264 & 0.029 \\
LRCVRMLE & 900 & 1.587 & 0.969 & 0.205 & 0.076 \\
\hline
MLE & 25  & 0.985 & 0.473 & 0.470 & 0.416 \\
MLE & 64  & 1.526 & 0.715 & 0.505 & 0.182 \\
MLE & 100 & 1.278 & 0.843 & 0.534 & 0.094 \\
MLE & 400 & 1.318 & 0.433 & 0.035 & 0.029 \\
MLE & 625 & 1.141 & 0.529 & 0.084 & 0.030 \\
MLE & 900 & 1.306 & 0.587 & 0.142 & 0.051 \\
\hline
VRMLE & 25  & 0.987 & 0.468 & 0.445 & 0.416 \\
VRMLE & 64  & 1.523 & 0.708 & 0.440 & 0.189 \\
VRMLE & 100 & 1.343 & 0.892 & 0.176 & 0.118 \\
VRMLE & 400 & 1.322 & 0.436 & 0.040 & 0.029 \\
VRMLE & 625 & 1.139 & 0.528 & 0.001 & 0.030 \\
VRMLE & 900 & 1.305 & 0.586 & 0.036 & 0.051 \\
\hline
\end{tabular}
\end{table}

\begin{figure}
    \centering
    \includegraphics[width=0.6\linewidth]{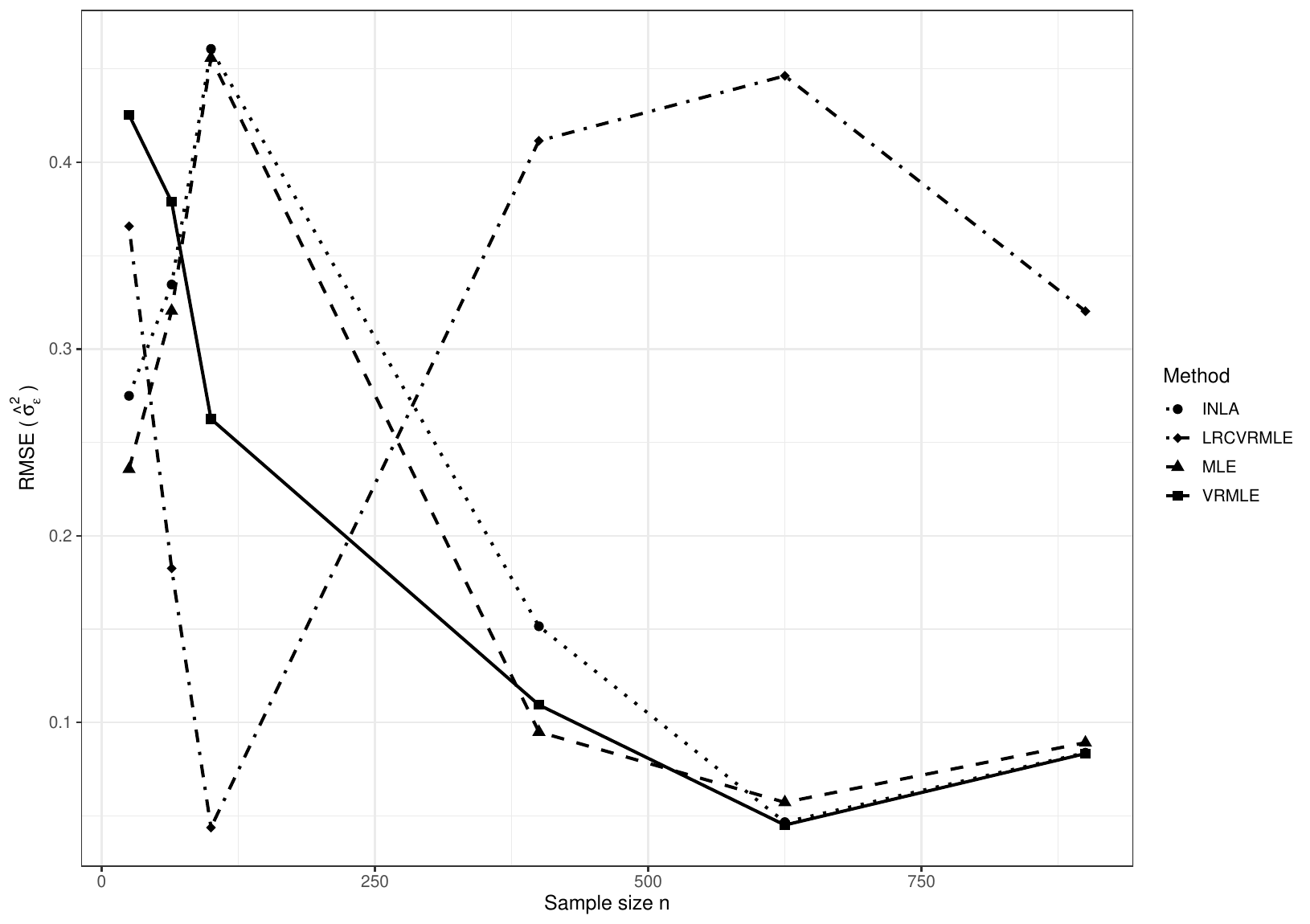}
   \caption{Average RMSE comparison for $\hat{\sigma}_\varepsilon^2$.}
    \label{fig:rmse_error}
\end{figure}

\begin{table}[ht]
\centering
\caption{Predictive performance comparison.}
  \label{tab:breast-comp}
\begin{tabular}{lcc}
\hline
  Method & RMSE    &   MAE\\
      \hline
     VRMLE & 0.5751874 & 0.4044629\\
LRCVRMLE & 0.7838234 & 0.6262440\\
 Exact REML & 0.5983507 & 0.4196303\\
 MLE & 0.5983492 & 0.4196293\\
     INLA & 0.5982228 & 0.4195322\\
\hline
\end{tabular}
\end{table}

\begin{figure}
    \centering
    \includegraphics[width=0.8\linewidth]{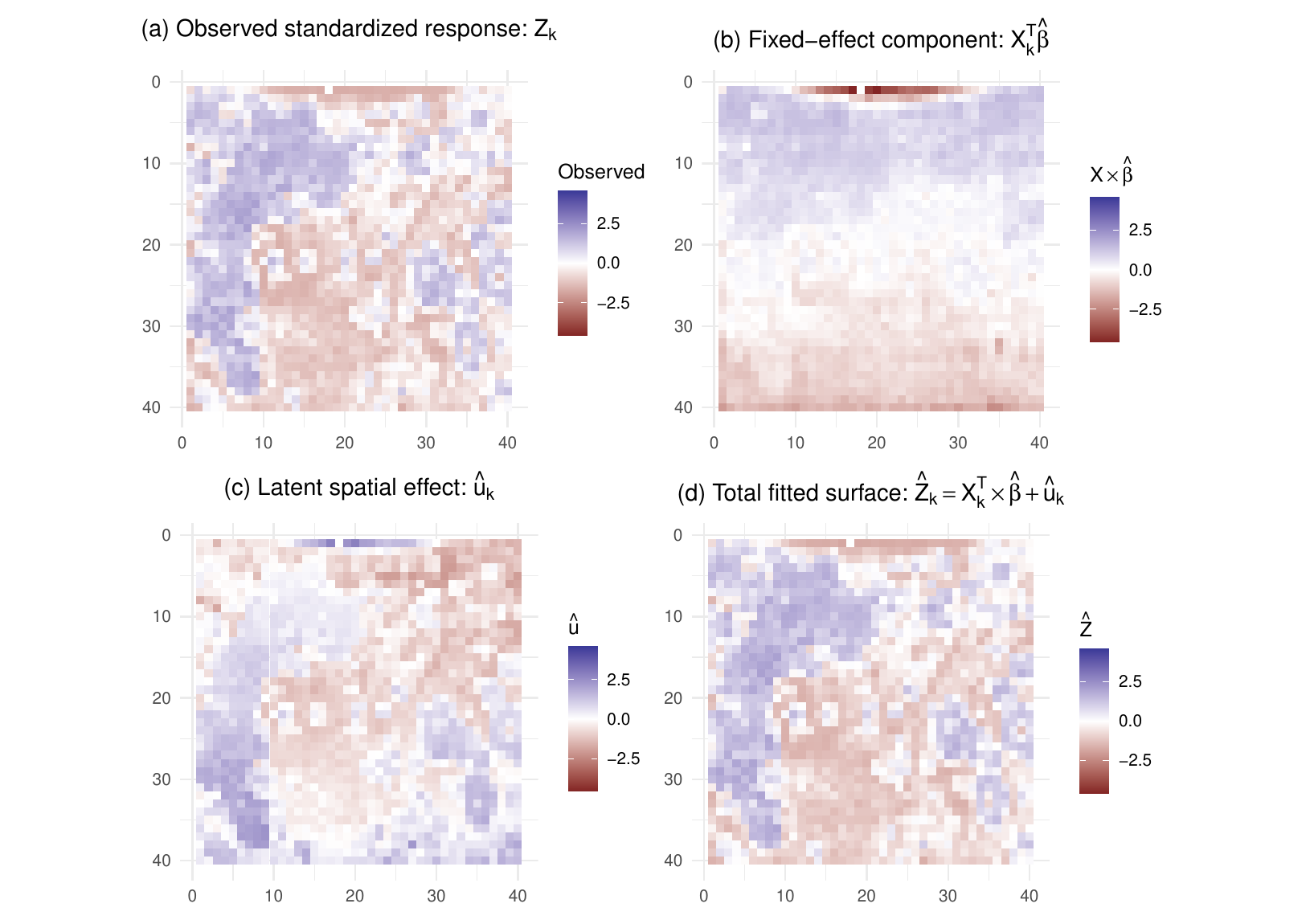}
    \caption{Spatial variation of observed and predicted EPCAM gene count.}
    \label{fig:result_epcam}
\end{figure}
In Figure~\ref{fig:mean_kl} we visualize the mean KL divergence between the variational posterior, $q(\bm u)$, and the true posterior $g_{\bm \theta}(\bm u | \bm Y)$ with increasing sample sizes. Figure~\ref{fig:mean_kl} illustrates that the variational posterior, $q(\bm u)$, becomes closer to the true posterior distribution, $g_{\bm \theta}(\bm u | \bm Y)$, with increasing sample size. Although there is a temporary increase in KL divergence that is detected at the intermediate spatial sampling units, the downward trend implies improved approximation accuracy and asymptotic empirical stability of the VRMLE estimator. In Table~\ref{tab:metrics_simulation}, we demonstrate the simulation performance. In Table~\ref{tab:metrics_simulation}, we observe that the average MSPE decreases with increasing sample size, which is competitive compared to INLA, MLE, and LRCVRMLE. A similar behavior is detectable for the MSE of the posterior mean and variance of the random effect. Table~\ref{tab:metrics_simulation} compares the $\beta-$surface accuracy, and that accuracy is quite competitive compared to the other methods. Thus, we can conclude that our VRMLE's predictive performance and other accuracy are quite competitive compared to the other traditional models. The RMSE of $\hat{\sigma}_u^2$ decreases with increasing the size of the spatial domain $\mathcal S_n$. In Figure~\ref{fig:rmse_error}, we illustrate the decreasing RMSE for error variance estimation with increasing sample size, and we observe that our VRMLE has performed effectively, dominating the other two methods.

For real data analysis, we consider the spatial transcriptomics data obtained from \textit{Xenium Human Breast Cancer}. This platform generally provides high-resolution gene expression measurements at the single-cell level along with spatial coordinates. Each observation corresponds to a cell located at spatial position $\mathbf{s}_i = (x_i, y_i) \in \mathbb{R}^2$ with associated gene expression counts. Let $\mathbf{s}_i = (x_i, y_i) \in \mathbb{R}^2$ denote the spatial location of cell $i$, for $i = 1, \dots, N$, where $N$ is the total number of observed cells. Let $\{(\mathbf{s}_i, Y_i, L_i)\}_{i=1}^{N}$ denote the observed data, where $Y_i$ represents the expression count of a selected gene, EPCAM and $L_i$ denotes the total library size for cell $i$. Due to the high density of cell-level observations, we aggregate the data onto a regular grid covering the spatial domain. Let $\{\mathcal{G}_k\}_{k=1}^{M}$ denote the collection of grid cells, where $M$ is the number of non-empty grid cells. For each grid cell $\mathcal{G}_k$, we compute:
\[
\begin{aligned}
   \bar{Y}_k = \frac{1}{m_k} \sum_{i \in \mathcal{G}_k} Y_i, \;
\bar{L}_k = \frac{1}{m_k} \sum_{i \in \mathcal{G}_k} L_i, \;
m_k = \#\{i : i \in \mathcal{G}_k\},\;
Z_k = \frac{\log(1 + \bar{Y}_k) - \overline{Y}}{\mathrm{sd}(\log(1 + \bar{Y}_k))}.
\end{aligned}
\]
and consider $Z_k$ as a response in the $k$th grid and $\bm{X}_k = \left(1, \log(1+\bar{L}_k), \log(1+n_k), s_{k1}, s_{k2}\right)$ is the design matrix corresponding to the $k$th grid with the center of location, $\bm s_k = (s_{k1}, s_{k2})$ and jointly estimate $(\tau_y, \tau_u)$ from \eqref{eq:elbo-reml}. We have considered the model for this spatial data as $\bm Z | \bm \beta, \bm u, \bm \tau_y\sim \mathcal N(\bm X^\top \bm \beta + \bm u, \tau_y^{-1} \bm I_n)$. We compare the predictive performance of our VRMLE with other methods using two metrics, RMSE and MAE, in Table~\ref{tab:breast-comp}. From Table~\ref{tab:breast-comp}, we observe that our VRMLE has the lowest RMSE and MAE among the five methods: REML, MLE, INLA, LRCVRMLE, and VRMLE. Figure~\ref{fig:result_epcam} illustrates the standardized spatial surface related to the EPCAM gene expression count, which is obtained using our proposed VRMLE model. In Figure~\ref{fig:result_epcam}, the panel (a) visualizes the standardized observed gene expression count, whereas panel (b) demonstrates the fixed effect $\bm {X} ^\top\widehat {\bm\beta}$, panel (c) demonstrates the latent spatial random effect, $\bm u$, and the last panel (d) illustrates the fitted response, $\widehat Z_k = \bm X_k^\top\widehat{\bm\beta} + \widehat u_k $. Panel (b) of Figure~\ref{fig:result_epcam} exhibits the large-scale spatial fixed effect, which is contributed by the sequencing depth, cell density, and spatial coordinates; however, panel (c) captures spatial residual dependence. In summary, panel (b) and panel (c) depict the meaningful contribution of the fixed effect and latent random effect, particularly that the fixed effect captures the spatial trend, whereas the ICAR spatial random effect accounts for residual dependence on a finer scale. These empirical results suggest that our proposed VRMLE substantially helps us in distinguishing the fixed effect and the random residual spatial effect. Thus, from Figure~\ref{fig:result_epcam}, we can observe that the observed and VRMLE-estimated spatial surfaces are very close to each other. From this empirical evidence, it is clear that for Gaussian ICAR modeling, our proposed approximated VRMLE approach provides better scalability in modeling the data. For better insight into the contribution of the feature space and spatial components, the fitted VRMLE surface can be decomposed into the fixed-effect, $\bm X_k^\top\widehat{\bm \beta}$, and a latent spatial random effect, $\widehat{\bm u}_k$. In this analysis, the variance of the fixed effect is $0.673$, and it explains approximately $68.6\%$ of the fitted variability. Similarly, the variability of the spatial random effect is $0.738$, and that explains approximately $75.2\%$ of the fitted variability. The correlation between the observed response, $Z_k$, and the fitted random effect, $\widehat{\bm u}_k$, is $0.62$, whereas that between the fixed effect and the fitted random effect is $-0.31$, which altogether suggests that the two components capture different sources of variation. Next, we assess the predictive contribution of the latent random spatial effect via conducting an ablation experiment. We have considered two models; one involves a random effect, $\bm u$, $ \text{Full Model:}\;\bm Z = \bm X^\top \bm \beta + \bm u + \bm \varepsilon$ and the other is the reduced model in the absence of a random effect, $ \text{Reduced Model:}\;\bm Z = \bm X^\top \bm \beta + \bm \varepsilon$. In each repetition, the observations are randomly split into training and test data by maintaining the split ratio of $80:20$, and this process is repeated five times. We fit the reduced model using the training data and predict for test locations as $\widehat{\bm Z}_{\mathrm{test}} = \bm X_{\mathrm{test}}^\top \widehat{\bm{\beta}}$. In the full VREML model, we estimate the latent random effect for training locations, but for held-out locations we use a conditional distribution using the fitted ICAR model. Suppose the ICAR precision matrix is partitioned as
\[
R=
\begin{pmatrix}
R_{TT} & R_{TS}\\
R_{ST} & R_{SS}
\end{pmatrix},
\]
where the subscripts $T$ and $S$ denote the training and test locations, respectively. We obtain the marginal precision matrix for the training locations through $R_{T,\mathrm{marg}} = R_{TT} - R_{TS} R_{SS}^{-1} R_{ST}$, and employ it at the time of fitting the VREML model. After obtaining the posterior estimates of the spatial random effects for training locations, we predict the latent spatial effects at the held-out locations through the conditional ICAR expectation, $\widehat{\bm u}_{\mathrm{test}} = - R_{SS}^{-1} R_{ST}\widehat{\bm u}_{\mathrm{train}}$. Then the prediction for the full model is $\widehat{\bm Z}_{\mathrm{test}} =
\bm X_{\mathrm{test}}^\top \widehat{\bm{\beta}} +\widehat{\bm  u}_{\mathrm{test}}$.
We first investigate the existence of spatial autocorrelation in the residuals for the reduced model by calculating Moran's $I$, and the value is $0.651$. This is a clear indication that the residuals of the reduced model are positively spatially correlated. This indicates that the observed covariates alone are insufficient to explain the spatial variation in the EPCAM gene expression data and therefore necessitates the inclusion of a latent spatial random effect in the proposed VREML model. Finally, we evaluate the predictive performance of the two competing models using RMSE, MAE, and MSE for the test data.
\begin{table}[h!]
    \centering
    \caption{Contribution to the predictive performance of random effects.}
    \begin{tabular}{cccc}
    \hline
    Model & Average RMSE (SD) & Average MAE (SD) & Average MSE (SD)\\
    \hline
     Full Model    &  0.454 (0.0248) & 0.342 (0.0091) & 0.207 (0.023)\\
      Reduced Model   & 0.769 (0.0369) & 0.645 (0.019) & 0.592 (0.0582)\\
      \hline
    \end{tabular}
    \label{tab:abl_sum}
\end{table}
Table~\ref{tab:abl_sum} summarizes the predictive performance of the two competing models over five-fold train-test splits. From Table~\ref{tab:abl_sum}, we observe that the reduced model has yielded higher average RMSE, MAE, and MSE, whereas the full model produces smaller values. Therefore, inclusion of a latent spatial random effect in the full VREML model improves average RMSE, MAE, and MSE by $41\%, 47\%$, and $65\%$. This finding in Table~\ref{tab:abl_sum} is a strong indication of the improvement in predictive performance by the inclusion of the latent ICAR component, which captures important residual spatial variation not explained by the observed covariates. Thus, our proposed VREML model makes a substantial contribution to out-of-sample predictive accuracy. These empirical results validate that both the fixed effect and random effect significantly explain the spatial variation of EPCAM gene expression.

Our proposed ICAR-based VRMLE approach is closely connected to \cite{kang2011bayesian}'s spatial random-effect model (SRE). Their method approximates the random effect via a multi-resolution basis function, not necessarily orthogonal, which is effectively capable of reducing the dimension. \cite{kang2011bayesian}'s low-rank approximation method replaces the computational burden of repeated inversion of the covariance matrix; however, our VRMLE makes the precision matrix sparse through encoding the neighborhood dependence structure of spatial areal units through ICAR dependence. Our LRCVRMLE makes a bridge between \cite{kang2011bayesian}'s SRE and our VRMLE via retaining the ICAR-based precision structure at the time of approximating dense posterior covariance through \eqref{eq:reparm-elbo-reml}. Thus, we can claim that our proposed framework may be viewed as complementary to the SRE model because, while SRE targets dimension-reduced spatial random effects, at the same time, our VRMLE focuses on scalable REML-type inference for ICAR random effects. Thus, we thoroughly demonstrate that our novel VRMLE provides competitive predictive performance compared to MLE and INLA. From a broader perspective, we can consider VRMLE as a potential alternative to the likelihood-based approaches for high-dimensional spatial models. The combination of the interpretability of REML along with the scalability of variational approaches motivates new directions for efficient inference for large spatial data sets in this article. In the near future, we will establish the asymptotic properties of the VRMLE estimator under different spatial asymptotics.
\section*{Declaration of competing interest} No competing interest.
\section*{Funding Availability} No funding options are available.
\section*{Data and Code Availability} Data is open source, and code is available \url{https://github.com/debjoythakur/Variational-REML-Spatial}.
\section*{Acknowledgement} The author is extremely thankful to the reviewer and associate editor for their constructive suggestion to improve the quality of the paper.
\printcredits

\bibliographystyle{cas-model2-names}

\bibliography{reference-reml_new}


\end{document}